\documentclass[letterpaper, 10 pt, conference]{ieeeconf}
\IEEEoverridecommandlockouts
\usepackage{color}
\usepackage[utf8]{inputenc}
\usepackage[english]{babel}
\usepackage{amsmath} 

\usepackage{amsthm}
\usepackage{amssymb}  
\usepackage{multicol}
\usepackage{lipsum}
\usepackage{xcolor}
\usepackage{graphicx}
\usepackage{subfig}
\usepackage{placeins}
\newtheorem{lem}{Lemma}
\newtheorem{thm}{Theorem}

\usepackage{nowidow}
\usepackage{flushend}

\newcommand{\R}{\mathbb{R}}

\newcommand{\E}{\mathbb{E}}

\newcommand{\setX}{{\cal X}}
\newcommand{\setU}{{\cal U}}

\title{Local Policy Optimization for Trajectory-Centric Reinforcement Learning}
\author{Patrik Kolaric$^1$, Devesh K. Jha$^2$, Arvind U. Raghunathan$^2$, Frank L. Lewis$^1$,\\ Mouhacine Benosman$^2$, Diego Romeres$^2$ Daniel Nikovski$^2$
	\thanks{$^1$UTA Research Institute, University of Texas at Arlington, Fort Worth, TX,USA (Email: \texttt{patrik.kolaric@mavs.uta.edu,
lewis@uta.edu }}
    \thanks{$^2$ Mitsubishi Electric Research Laboratories (MERL), Cambridge, MA, USA. Email: \texttt{\{jha,raghunathan,benosman,romeres,nikovski\}@merl.com}}%
	}

\begin{document}

\maketitle
\thispagestyle{empty}
\pagestyle{empty}

\begin{abstract}
The goal of this paper is to present a method for simultaneous trajectory and local stabilizing policy optimization to generate local policies for trajectory-centric model-based reinforcement learning (MBRL). This is motivated by the fact that global policy optimization for non-linear systems could be a very challenging problem both algorithmically and numerically. However, a lot of robotic manipulation tasks are trajectory-centric, and thus do not require a global model or policy. Due to inaccuracies in the learned model estimates, an open-loop trajectory optimization process mostly results in very poor performance when used on the real system.  Motivated by these problems, we try to formulate the problem of trajectory optimization and local policy synthesis as a single optimization problem. It is then solved simultaneously as an instance of nonlinear programming. We provide some results for analysis as well as achieved performance of the proposed technique under some simplifying assumptions.
\end{abstract}
\section{Introduction}\label{sec:introduction}
 Reinforcement learning (RL) is a learning framework that addresses sequential decision-making problems, wherein an `agent' or a decision maker learns a policy to optimize a long-term reward by interacting with the (unknown) environment. At each step, the RL agent obtains evaluative feedback (called reward or cost) about the performance of its action, allowing it to improve the performance of subsequent actions \cite{sutton1998reinforcement,vrabie2013optimal}. Although RL has witnessed huge successes in recent times~\cite{silver2016mastering,silver2017mastering}, there are several unsolved challenges, which restrict the use of these algorithms for industrial systems. In most practical applications, control policies must be designed to satisfy operational constraints, and a satisfactory policy should be learnt in a data-efficient fashion~\cite{vamtutorial}.

Model-based reinforcement learning (MBRL) methods~\cite{deisenroth2011pilco} learn a model from exploration data of the system, and then exploit the model to synthesize a trajectory-centric controller for the system~\cite{levine2013guided}. These techniques are, in general, harder to train, but could achieve good data efficiency~\cite{levine2016end}. Learning reliable models is very challenging for non-linear systems and thus, the subsequent trajectory optimization could fail when using  inaccurate models. However, modern machine learning methods such as Gaussian processes (GP), stochastic neural networks (SNN), etc. can generate uncertainty estimates associated with predictions~\cite{rasmussen2003gaussian,romeres2019semiparametrical}. These uncertainty estimates could be used to estimate the confidence set of system states at any step along a given controlled trajectory for the system. The idea presented in this paper considers the stabilization of the trajectory using a local feedback policy that acts as an attractor for the system in the known region of uncertainty along the trajectory~\cite{tedrake2010lqr}.

We present a method for simultaneous trajectory optimization and local policy optimization, where the policy optimization is performed in a neighborhood (local sets) of the system states along the trajectory. These local sets could be obtained by a stochastic function approximator (e.g., GP, SNN, etc.) that used to learn the forward model of the dynamical system. The local policy is obtained by considering the worst-case deviation of the system from the nominal trajectory at every step along the trajectory. Performing simultaneous trajectory and policy optimization could allow us to exploit the modeling uncertainty as it drives the optimization to regions of low uncertainty, where it might be easier to stabilize the trajectory. This allows us to constrain the trajectory optimization procedure to generate robust, high-performance controllers. The proposed method automatically incorporates state and input constraints on the dynamical system.

\textbf{Contributions.} The main contributions of the current paper are:
\begin{enumerate}
    \item We present a novel formulation of simultaneous trajectory optimization and time-invariant local policy synthesis for stabilization.
    \item We present analysis of the proposed technique  that allows us to analytically derive the gradient of the robustness constraint for the optimization problem.
\end{enumerate}
It is noted that this paper only presents the controller synthesis part for MBRL -- a more detailed analysis of the interplay between model uncertainties and controller synthesis is deferred to another publication.

\section{Related Work}\label{sec:related_work}
MBRL has raised a lot of interest recently in robotics applications, because model learning algorithms are largely task independent and data-efficient~\cite{2019arXiv190702057W, levine2016end, deisenroth2011pilco}. However, MBRL techniques are generally considered to be hard to train and likely to result in poor performance of the resulting policies/controllers, because the inaccuracies in the learned model could guide the policy optimization process to low-confidence regions of the state space. For non-linear control, the use of trajectory optimization techniques such as differential dynamic programming~\cite{jacobson1968new} or its first-order approximation, the iterative Linear Quadratic Regulator (iLQR)~\cite{tassa2012synthesis} is very popular, as it allows the use of gradient-based optimization, and thus could be used for high-dimensional systems. As the iLQR algorithm solves the local LQR problem at every point along the trajectory, it also computes a sequence of feedback gain matrices to use along the trajectory. However, the LQR problem is not solved for ensuring robustness, and furthermore the controller ends up being time-varying, which makes its use somewhat inconvenient for robotic systems. Thus, we believe that the controllers we propose might have better stabilization properties, while also being time-invariant.

Most model-based methods use a function approximator to first learn an approximate model of the system dynamics, and then use stochastic control techniques to synthesize a policy. Some of the seminal work in this direction could be found in~\cite{levine2016end, deisenroth2011pilco}. The method proposed in~\cite{levine2016end} has been shown to be very effective in learning trajectory-based local policies by sampling several initial conditions (states) and then fitting a neural network which can imitate the trajectories by supervised learning. This can be done by using ADMM~\cite{boyd2011distributed} to jointly optimize trajectories and learn the neural network policies. This approach has achieved impressive performance on several robotic tasks~\cite{levine2016end}. The method has been shown to scale well for systems with higher dimensions. Several different variants of the proposed method were introduced later~\cite{chebotar2017path,montgomery2016guided, nagabandi2018neural}. However, no theoretical analysis could be provided for the performance of the learned policies. 

Another set of seminal work related to the proposed work is on the use of sum-of-square (SOS) programming methods for generating stabilizing controller for non-linear systems~\cite{tedrake2010lqr}. In these techniques, a stabilizing controller, expressed as a polynomial function of states, for a non-linear system is generated along a trajectory by solving an optimization problem to maximize its region of attraction~\cite{majumdar2013control}. 

Some other approaches to trajectory-centric policy optimization could be found in~\cite{theodorou2010generalized}. These techniques use path integral optimal control with parameterized policy representations such as dynamic movement primitives (DMPs)~\cite{ijspeert2013dynamical} to learn efficient local policies~\cite{williams2017information}. However, these techniques do not explicitly consider the local sets where the controller robustness guarantees could be provided, either. Consequently, they cannot exploit the structure in the model uncertainty. A workshop version of our paper could be found here~\cite{jha2019robust}.

\section{Problem Formulation}\label{sec:problem_definition}
In this section, we describe the problem studied in the rest of the paper. To perform trajectory-centric control, we propose a novel formulation for simultaneous design of open-loop trajectory and a time-invariant, locally stabilizing controller that is robust to bounded model uncertainties and/or system measurement noise. As we will present in this section, the proposed formulation is different from that considered in the literature in the sense it allows us to exploit sets of possible deviation of a system to design stabilizing controller. 


\subsection{Trajectory Optimization as Non-linear Program}

Consider the discrete-time dynamical system 
\begin{align}
	x_{k+1} 	&=	f(x_k,u_k)			\label{dynamic_diff} 
\end{align}
where $x_k \in \R^{n_x}$, $u_k \in \R^{n_u}$ are the differential 
states and controls, respectively.  The function $f : \R^{n_x+n_u} \rightarrow 
\R^{n_x}$ governs the evolution of the differential states.  
Note that the discrete-time formulation~\eqref{dynamic_diff} can be obtained from 
a continuous time system $\dot{x} = \hat{f}(x,u)$ by using the \emph{explicit Euler} integration 
scheme $(x_{k+1} - x_k) = \Delta t \hat{f}(x_k,u_k)$ where $\Delta t$ is the time-step for  
integration.  


For clarity of exposition we have limited our focus to discrete-time dynamical 
systems of the form in~\eqref{dynamic_diff} although the techniques we describe can be easily extended to implicit discretization 
schemes.  

In typical applications the states and controls are restricted to lie in sets 
$\setX := \{ x \in \R^{n_x} \,|\, \underline{x} \leq x \leq \overline{x} \} \subseteq \R^{n_x}$ and 
$\setU := \{ u \in \R^{n_u} \,|\, \underline{u} \leq u \leq \overline{u} \} \subseteq \R^{n_u}$, \emph{i.e.}
$x_k \in \setX, u_k \in \setU$. 
We use $[K]$ to denote the index set $\{0,1,\dots, K\}$. 
Further, there may exist nonlinear inequality constraints of the form
\begin{equation}
	g(x_k) \geq 0 \label{stateineq}
\end{equation}
with $g : \R^{n_x} \rightarrow \R^{m}$.  The inequalities in~\eqref{stateineq} are termed as \emph{path constraints}.
The trajectory optimization problem is to manipulate the controls $u_k$ over a certain number of time steps $[T-1]$ 
so that the resulting trajectory $\{x_k\}_{k \in [T]}$  minimizes a cost function 
$c(x_k,u_k)$.  Formally, we aim to solve the \emph{trajectory optimization problem}
\begin{equation}
	\begin{aligned}
		\min\limits_{x_k,u_k}		&\,	\sum\limits_{k \in [T]} c(x_k,u_k) \\
		\text{s.t.} 	&\,	\text{Eq.}~\eqref{dynamic_diff}-\eqref{stateineq} \text{ for } k \in [T] \\
				&\,	x_0  = \tilde{x}_0 \\
				&\,	x_{k} \in \setX \text{ for } k \in [T] \\
				&\,	u_k \in \setU \text{ for } k \in [T-1] 
	\end{aligned}\tag{TrajOpt}\label{trajopt}
\end{equation}
where $\tilde{x}_0$ is the differential state at initial time $k = 0$.  Before introducing the main problem of interest, we would like to introduce some notations. 

In the following text, we use the following shorthand notation, $||v||^2_M=v^TMv$. We denote the nominal trajectory  as $X \equiv x_0, x_1, x_2, x_3,\dots,x_{T-1},x_T$, $U \equiv u_0, u_1, u_2, u_3, ..., u_{T-1}$. The actual trajectory followed by the system is denoted as $\hat X \equiv \hat x_0, \hat x_1, \hat x_2, \hat x_3,\dots,\hat x_{T-1},\hat x_T$. We denote a local policy as $\pi_W$, where $\pi$ is the policy and $W$ denotes the parameters of the policy. The trajectory cost is also sometimes denoted as $J= \sum\limits_{k \in [T]} c(x_k,u_k)$.

\subsection{Trajectory Optimization with Local Stabilization}\label{sec:robust_problem_def}
This subsection introduces the main problem of interest in this paper. A schematic of the problem studied in the paper is also shown in Figure~\ref{fig:problem_definition}. In the rest of this section, we will describe how we can simplify the trajectory optimization and local stabilization problem and turn it into an optimization problem that can be solved by standard non-linear optimization solvers.
\begin{figure}
	\flushleft
	\includegraphics[width=1.0\columnwidth]{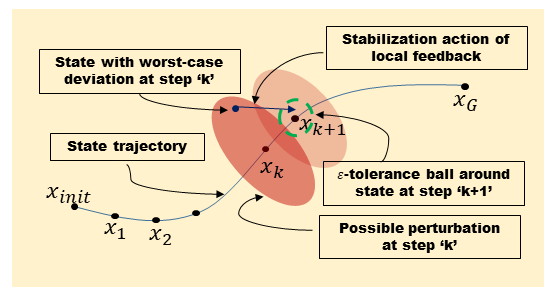}
	\caption{A schematic representation of the time-invariant local control introduced in this paper.}
	\label{fig:problem_definition}
\vspace{-7 mm}
\end{figure}

Consider the case where the system dynamics, $f$ is only partially known, and the known component of $f$ is used to design the controller. Consider the deviation of the system at any step '$k$' from the state trajectory $X$ and denote it as $\delta x_k \equiv x_k - \hat {x}_k$. We introduce a local (time-invariant) policy $\pi_W$ that regulates the local trajectory deviation $\delta x_k$ and thus, the final controller is denoted as $\hat u_k = {u}_k + \pi_W(\delta x_k)$. The closed-loop dynamics for the system under this control is then given by the following:
\begin{equation}\label{eqn:closed_loop_dynamics}
\hat x_{k+1} = f(\hat x_k, \hat u_k)=f({x}_k + \delta x_k, {u}_k + \pi_W(\delta x_k))
\end{equation}
The main objective of the paper is to find the time-invariant feedback policy $\pi_W$ that can stabilize the open-loop trajectory $X$ locally within $\mathbb{R}_k \subset \R^{n_x}$ where $\mathbb{R}_k$ defines the set of uncertainty for the deviation $\delta x_k$. The uncertainty region $\R_k$ can be approximated by fitting an ellipsoid to the uncertainty estimate using a diagonal positive definite matrix $S_k$ such that $\R_k = \{\delta x_k : \delta x_k^T S_k \delta x_k \leq  1\}$. The general optimization problem that achieves that is proposed as:
\begin{equation}\label{generalFormulation}
\begin{aligned}
J^* = & \underset{{U}, {X}, W}{min}\; \underset{\delta x_k \in \R_k}{\E}[ J(X+\delta X, U+\pi_W(\delta x_k)]\\
&{x}_{k+1} = \hat{f}({x}_k, {u}_k)\\ 
\end{aligned}
\end{equation}
where $\hat f(\cdot,\cdot)$ denotes the known part of the model. Note that in the above equation, we have introduced additional optimization parameters corresponding to the policy $\pi_W$ when compared to \ref{trajopt} in the previous section. However, to solve the above, one needs to resort to sampling in order to estimate the expected cost. Instead we introduce a constraint that solves for the worst-case cost for the above problem.

\textbf{Robustness Certificate.} The robust trajectory optimization problem is to minimize the trajectory cost while at the same time satisfying a \textit{robust constraint} at every step along the trajectory. This is also explained in Figure~\ref{fig:problem_definition}, where the purpose of the local stabilizing controller is to push the max-deviation state at every step along the trajectory to $\epsilon$-tolerance balls around the trajectory. Mathematically, we express the problem as following:
\begin{equation}
	\begin{aligned}
		\min\limits_{x_k,u_k,W}		&\,	\sum\limits_{k \in [T]} c(x_k,u_k) \\
		\text{s.t.} 	&\,	\text{Eq.}~\eqref{dynamic_diff}-\eqref{stateineq} \text{ for } k \in [T] \\
				&\,	x_0  = \tilde{x}_0 \\
				&\,	x_{k} \in \setX \text{ for } k \in [T] \\
				&\,	u_k \in \setU \text{ for } k \in [T-1] \\
				&\,	\max\limits_{\delta x_k \in \R_k} ||x_{k+1}-f(x_k+\delta x_k, u_k+\pi_W(\delta x_k))||_2\leq \epsilon_k 
	\end{aligned}\tag{RobustTrajOpt}\label{robusttrajopt}
\end{equation}

The additional constraint introduced in~\ref{robusttrajopt} allows us to ensure stabilization of the trajectory by estimating parameters of the stabilizing policy $\pi_W$. It is easy to see that~\ref{robusttrajopt} solves the worst-case problem for the optimization considered in~\eqref{generalFormulation}. However,~\ref{robusttrajopt} introduces another hyperparameter to the optimization problem, $\epsilon_k$. In the rest of the paper, we refer to the following constraint as the \textit{robust constraint}:
\begin{equation}\label{eqn:robustconstraint}
\max\limits_{\delta x_k^T S_k \delta x_k \leq  1} ||x_{k+1}-f(x_k+\delta x_k, u_k+\pi_W(\delta x_k))||_2\leq \epsilon_k
\end{equation}
Solution of the \textit{robust constraint} for generic non-linear system is out of scope of this paper. Instead, we linearize the trajectory deviation dynamics as shown in the following Lemma.
\begin{lem}
The trajectory deviation dynamics $\delta x_{k+1} = x_{k+1}  - \hat{x}_{k+1}$ approximated locally around the optimal trajectory $(X,U)$ are given by
\begin{equation}
\begin{aligned}
\delta x_{k+1} &= A({x}_k, {u}_k)\cdot \delta x_k + B({x}_k, {u}_k)\cdot \pi_W(\delta x_k)\\
A({x}_k, {u}_k) &\equiv \nabla_{ {x}_k} \hat{f}({x}_k, {u}_k)\\
B({x}_k, {u}_k) &\equiv  \nabla_{{u}_k} \hat{f}({x}_k, {u}_k)
\end{aligned}
\end{equation}
\end{lem}
\begin{proof}
Use Taylor's series expansion to obtain the desired expression. 
\end{proof}

To ensure feasibility of the~\ref{robusttrajopt} problem and avoid tuning the hyperparameter $\epsilon_k$, we make another relaxation by removing the \textit{robust constraint} from the set of constraints and move it to the objective function. Thus, the simplified robust trajectory optimization problem that we solve in this paper can be expressed as following (we skip the state constraints to save space).
\begin{equation}
	\begin{aligned}
		\min\limits_{x_k,u_k,W}		&\,	(\sum\limits_{k \in [T]} c(x_k,u_k) +\alpha\sum\limits_{k \in [T]}d_{max,k})\\
		\text{s.t.} 	&\,	\text{Eq.}~\eqref{dynamic_diff}-\eqref{stateineq} \text{ for } k \in [T] \\
	\end{aligned}\tag{RelaxedRobustTrajOpt}\label{relaxedrobusttrajopt}
\end{equation}
where the term $d_{max,k}$ is defined as following after linearization.
\begin{equation}
\begin{aligned}
&d_{max,k} \equiv  \\ 
&\max_{\delta x_k^T S_k \delta x_k \leq  1} ||A({x}_k, {u}_k)\cdot \delta x_k + B({x}_k, {u}_k)\cdot \pi_W(\delta x_k)||_P^2
\end{aligned}\label{dmax}
\end{equation}
Note that the matrix $P$ allows to weigh states differently. In the next section, we present the solution approach to compute the gradient for the~\ref{relaxedrobusttrajopt} which is then used to solve the optimization problem. Note that rthis esults in simultaneous solution to open-loop and the stabilizing policy $\pi_W$.
\section{Solution Approach}\label{sec:Solution_approach}
This section introduces the main contribution of the paper, which is a local feedback design that regulates the deviation of an executed trajectory from the optimal trajectory generated by the optimization procedure. 


To solve the optimization problem presented in the last section, we first need to obtain the gradient information of the robustness heuristic that we introduced. However, calculating the gradient of the robust constraint is not straightforward, because the $max$ function is non-differentiable. The gradient of the robustness constraint is computed by the application of Dankins Theorem~\cite{bertsekas1997nonlinear}, which is stated next. \\
\textbf{Dankin's Theorem:} Let $K \subseteq \R^m$ be a nonempty, closed set and let $\Omega \subseteq  \R^n$ be a nonempty, open set. Assume that the function 
$f : \Omega \times K \rightarrow \R$ is continuous on $\Omega \times K$ 
and that $\nabla_x f(x, y)$ exists and is continuous on 
$\Omega \times K$. Define the function 
$g : \Omega \rightarrow \R \cup \{\infty\}$ by
\[
g(x) \equiv \sup\limits_{y \in K} f(x, y), x \in \Omega
\]
and 
\[
M(x) \equiv \{y \in K \,|\, g(x) = f(x, y) \}.
\]
Let $x \in \Omega$ be a given vector. Suppose that a neighborhood 
${\cal N}(x) \subseteq \Omega$ of $x$ 
exists such that $M(x')$ is nonempty for all $x' \in {\cal N}(x)$ 
and the set $\cup_{x' \in {\cal N}(x)} M(x')$ is bounded. 
The following two statements (a) and (b) are valid.
\begin{enumerate}[(a)]
    \item The function $g$ is directionally differentiable at $x$ 
    and 
    \[
        g'(x;d) = \sup_{y \in M(x)} \nabla_x f(x, y)^T d.
    \]
    \item If $M(x)$ reduces to a singleton, say $M(x) = \{y(x)\}$, 
    then $g$ is G$\hat{\text{a}}$eaux differentiable at $x$ and
    \[
        \nabla g(x) = \nabla_x f(x, y(x)).
    \]
    \end{enumerate}
    
\textbf{Proof} See~\cite{FacchineiPangVol2}, Theorem $10.2.1$.

Dankin's theorem allows us to find the gradient of the robustness constraint by first computing the argument of the maximum function and then evaluating the gradient of the maximum function at the point. Thus, in order to find the gradient of the robust constraint~\eqref{eqn:robustconstraint}, it is necessary to interpret it as an optimization problem in $\delta x_k$, which is presented next. In Section~\ref{sec:robust_problem_def}, we presented a general formulation for the stabilization controller $\pi_W$, where $W$ are the parameters that are obtained during optimization. However, solution of the general problem is beyond the scope of the current paper. Rest of this section considers a linear $\pi_W$ for analysis. \\
\begin{lem}\label{lem: QCQP}
Assume the linear feedback $\pi_W(\delta x_k) = W \delta x_k$. Then, the constraint (\ref{dmax})  is quadratic in $\delta x_k$, 
\begin{equation}\label{eqn:QCQP}
\begin{aligned}
& \max_{\delta x_k} ||M_k \delta x_k||_P^2 = \max_{\delta x_k} \delta x_k^TM_k^T\cdot P \cdot M_k \delta x_k
\\
& s.t. \quad \delta x_k^T S_k \delta x_k \leq 1
\end{aligned}
\end{equation} where $M_k$is shorthand notation for
\begin{equation}
\begin{aligned}
M_k({x}_k, {u}_k, W) \equiv  A({x}_k, {u}_k) + B({x}_k, {u}_k) \cdot W
\end{aligned}
\end{equation}
\end{lem}
\begin{proof}
Write $d_{max}$ from (\ref{dmax}) as the optimization problem
\begin{equation}
\begin{aligned}
& d_{max} =\\ 
&\max_{\delta x_k} ||A({x}_k, {u}_k)\cdot \delta x_k + B({x}_k, {u}_k)\cdot \pi_W(\delta x_k)||_P^2
\\
& s.t. \quad  \delta x_k^T S_k \delta x_k \leq  1
\end{aligned}
\end{equation}
Introduce the linear controller and use the shorthand notation for $M_k$ to write (\ref{eqn:QCQP}).
\end{proof}

The next lemma is one of the main results in the paper. It connects the robust trajectory tracking formulation~\ref{relaxedrobusttrajopt} with the optimization problem that is well known in the literature.

\begin{lem}\label{lem:max_eigenvalue}
The worst-case measure of deviation $d_{max}$ is 
\begin{equation}\label{max_cstr}
\begin{aligned}
 &d_{max} = 
 \\
 & \lambda_{max}(S_k^{-\frac{1}{2}}M_k^T\cdot P \cdot M_kS_k^{-\frac{1}{2}}) = ||P^{\frac{1}{2}}M_kS_k^{-\frac{1}{2}}||_2^2
\end{aligned}\nonumber
\end{equation}\nonumber
where $\lambda_{max}(\cdot)$ denotes the  maximum eigenvalue of a matrix and  $||\cdot ||_2$  denotes the spectral norm of a matrix.
Moreover, the worst-case deviation $\delta_{max}$ is the corresponding maximum eigenvector
\begin{equation}
\begin{aligned}
&\delta_{max} = \\ &\{\delta x_k :  \Big [S_k^{-\frac{1}{2}}M_k^T\cdot P \cdot M_kS_k^{-\frac{1}{2}}\Big  ] \cdot \delta x_k = d_{max} \cdot \delta x_k\}
\end{aligned}
\end{equation} 
\end{lem}
\begin{proof} 
Apply coordinate transformation $\delta \tilde{x}_k = {S_k}^{\frac{1}{2}}\delta x_k$ in (\ref{eqn:QCQP}) and write
\begin{equation}\label{lin_const2}
\begin{aligned}
& \max_{\delta \tilde{x}_k} \delta \tilde{x}_kS_k^{-\frac{1}{2}}M_k^T\cdot P \cdot M_k S_k^{-\frac{1}{2}}\delta \tilde{x}_k
\\
& s.t. \quad \delta \tilde{x}_k \delta \tilde{x}_k \leq 1
\end{aligned}
\end{equation}
Since $S_k^{-\frac{1}{2}}M_k^T \cdot P \cdot M_k S_k^{-\frac{1}{2}}$ is positive semi-definite, the maximum lies on the boundary of the set defined by the inequality. Therefore, the problem is equivalent to \begin{equation}\label{lin_cost3}
\begin{aligned}
& \max_{\delta \tilde{x}_k} \delta \tilde{x}_kS_k^{-\frac{1}{2}}M_k^T\cdot P \cdot M_k S_k^{-\frac{1}{2}}\delta \tilde{x}_k
\\
& s.t. \quad \delta \tilde{x}_k \delta \tilde{x}_k = 1
\end{aligned}
\end{equation}The formulation (\ref{lin_cost3}) is a special case with a known analytic solution. Specifically, the maximizing deviation $\delta_{max}$ that solves (\ref{lin_cost3}) is the maximum eigenvector of $S_k^{-\frac{1}{2}}M_k^T\cdot P \cdot M_k S_k^{-\frac{1}{2}}$, and the value $d_{max}$ at the optimum is the corresponding eigenvalue.   
\end{proof}

This provides us with the maximum deviation along the trajectory at any step 'k', and now we can use Danskin's theorem to compute the gradient which is presented next.

\begin{thm}\label{thm:gradient_computation}
Introduce the following notation, $\mathcal{M}(z) = S_k^{-\frac{1}{2}}M_k^T(z)\cdot P \cdot M_k(z) S_k^{-\frac{1}{2}}$. The gradient of the robust inequality constraint $d_{max}$ with respect to an arbitrary vector $z$ is
\begin{equation}
\begin{aligned}
\nabla_z d_{max}= \nabla_z \delta_{max}^T\mathcal{M}(z)\delta_{max}
\end{aligned}\nonumber
\end{equation} Where $\delta_{max}$ is maximum trajectory deviation introduced in Lemma 3.
\end{thm}
\begin{proof}
Start from the definition of gradient of robust constraint
\begin{equation}
\begin{aligned}
\nabla_z d_{max} =\nabla_z\max_{\delta \tilde{x}_k} \delta \tilde{x}_k\mathcal{M}(z)\delta \tilde{x}_k
\end{aligned}\nonumber 
\end{equation} Use Danskin's Theorem and the result from Lemma~\ref{lem:max_eigenvalue} to write the gradient of robust constraint with respect to an arbitrary $z$,
\begin{equation}
\begin{aligned}
\nabla_z d_{max}= \nabla_z \delta_{max}^T\mathcal{M}(z)\delta_{max}
\end{aligned}\nonumber
\end{equation}which completes the proof. 
\end{proof}

The gradient computed from Theorem~\ref{thm:gradient_computation} is used in solution of the~\ref{relaxedrobusttrajopt}-- however, this is solved only for a linear controller. The next section shows some results in simulation and on a real physical system.



\section{Experimental Results}\label{sec:results}
In this section, we present some results using the proposed algorithm for an under-actuated inverted pendulum, as well as on a experimental setup for a ball-and-beam system. We use a Python wrapper for the standard interior point solver IPOPT to solve the optimization problem discussed in previous sections. We perform experiments to evaluate the following questions:
\begin{enumerate}
    \item Can an off-the-shelf optimization solver find feasible solutions to the joint optimization problem described in the paper?
    \item Can the feedback controller obtained by this optimization stabilize the open-loop trajectory in the presence of bounded uncertainties?
    \item How good is the performance of the controller on a physical system with unknown system parameters ?
\end{enumerate}
In the following sections, we try to answer these questions using simulations as well as experiments on real systems.
\subsection{Simulation Results for Underactuated Pendulum}\label{sec:results_simulation}

The objective of this subsection is twofold: first, to provide insight into the solution of the optimization problem; and second, to demonstrate the effectiveness of that solution in the stabilization of the optimal trajectory.
\begin{figure}[!htb] 
	\flushleft
	\includegraphics[width=\columnwidth]{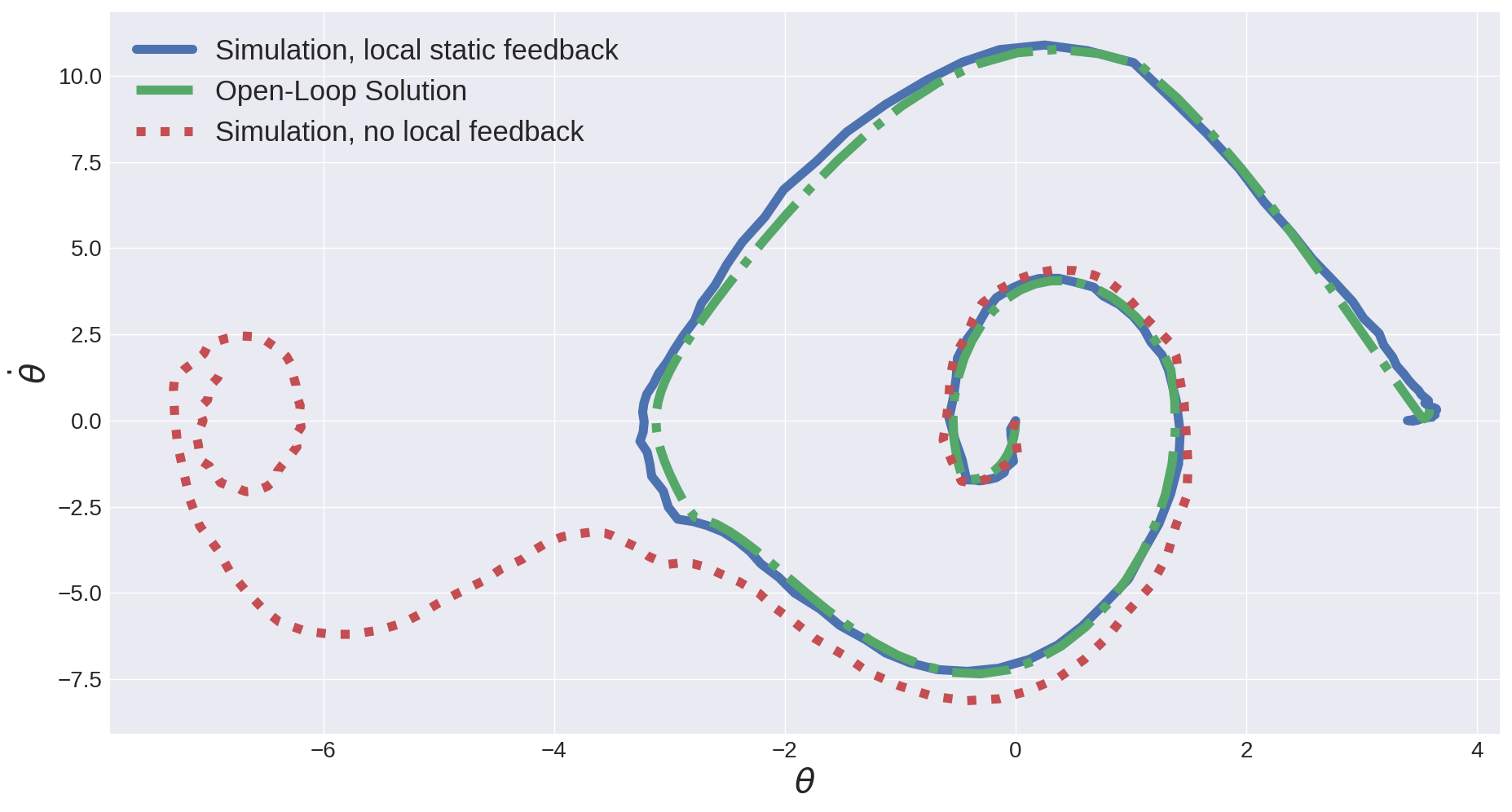}
	\caption{State-space representation of the optimal trajectory (green), stable closed-loop system using the obtained solution (blue), and unstable open-loop trajectory without the local feedback (red). Note that the feedback is time-invariant.}
	\label{fig:state_space}
\vspace{-1 mm}
\end{figure}
For clarity of presentation, we use an underactuated pendulum system, where trajectories can be visualized in state space. The dynamics of the pendulum is modeled as $I\ddot{\theta}+b\dot{\theta} + mgl\cdot sin(\theta) = u$. The continuous-time model is discretized as $(\theta_{k+1}, \dot{\theta}_{k+1}) = f((\theta_k, \dot{\theta}_k),u_k) $. The goal state is $x_g=[\pi, 0]$, and the initial state is $x_0=[0,0]$ and the control limit is $u \in [-1.7, 1.7]$. The cost is quadratic in both the states and input. The initial solution provided to the controller is trivial (all states and control are $0$). The number of discretization points along the trajectory is $N=120$, and the discretization time step is $\Delta t = 1/30$. The cost weight on robust slack variables is selected to be $\alpha = 10$. The uncertainty region is roughly estimated as $x_k^T\begin{bmatrix}
    1.0 & 0.0\\
    0.0 & 5.5
\end{bmatrix}x_k<1$ along the trajectory. Detailed analysis on uncertainty estimation based on Gaussian processes is deferred to future work, due to space limits. The optimization procedure terminates in $50$ iterations with the static solution $W=[-2.501840, -7.38725]$.

The controller generated by the optimization procedure is then tested in simulation, with noise added to each state of the pendulum model at each step of simulation as $x_{k+1} = f(x_k,u_k) + \omega$ with $\omega_\theta \sim \mathcal{U}(-0.2rad, 0.2rad])$ and $\omega_{\dot{\theta}} \sim \mathcal{U}(-0.05rad/s, 0.05rad/s])$. 


We tested the controller over several settings and found that the underactuated setting was the most challenging to stabilize. In Figure~\ref{fig:state_space}, we show the state-space trajectory for the controlled (underactuated) system with additional noise as described above. As seen in the plot, the open-loop controller becomes unstable with the additional noise, while the proposed controller can still stabilize the whole trajectory. 
\begin{figure}[!htb] 
	\flushleft
	\includegraphics[width=\columnwidth]{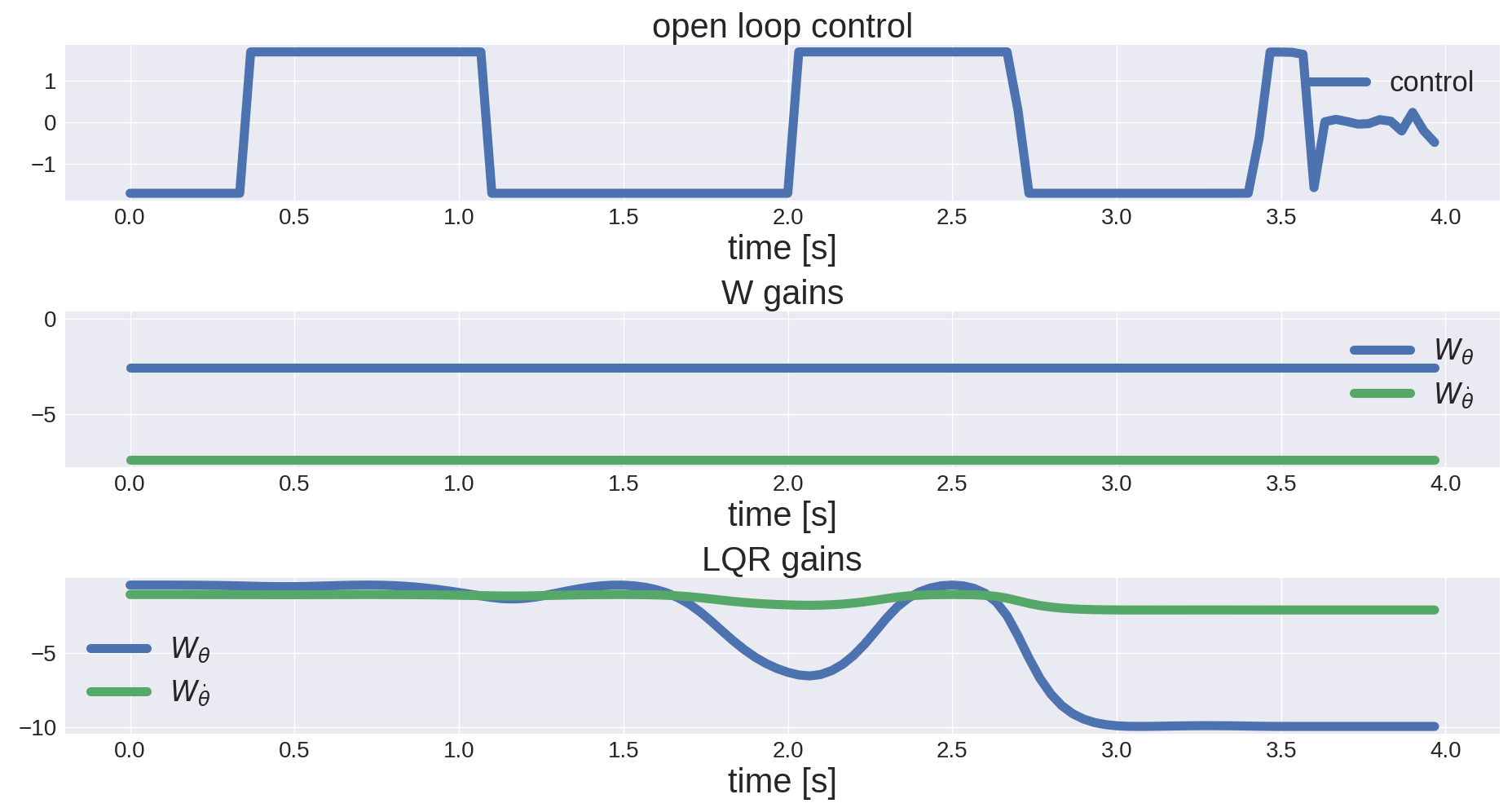}
	\caption{a) open-loop control b) static feedback matrix obtained from optimization c) LQR feedback}
	\label{fig:control_pendulum}
\vspace{-2 mm}
\end{figure}
\FloatBarrier
In Figure~\ref{fig:control_pendulum}, we show the control inputs, the time-invariant feedback gains obtained by the optimization problem. We also the time-varying LQR gains obtained along the trajectory to show provide some insight between the two solutions. As the proposed optimization problem is finding the feedback gain for the worst-case deviation from the trajectory, the solutions are different than the LQR-case. Next, in Figure~\ref{fig:error_pendulum}, we plot the error statistics for the controlled system (in the underactuated setting) over $2$ different uncertainty balls using each $12$ sample for each ball. We observe that the steady-state error goes to zero and the closed-loop system is stable along the entire trajectory. As we are using a linear approximation of the system dynamics, the uncertainty sets are still small, however the results are indicating that incorporating the full dynamics during stabilization could allow to generate much bigger basins of attraction for the stabilizing controller.

\begin{figure}[!htb] 
	\flushleft
	\includegraphics[width=\columnwidth]{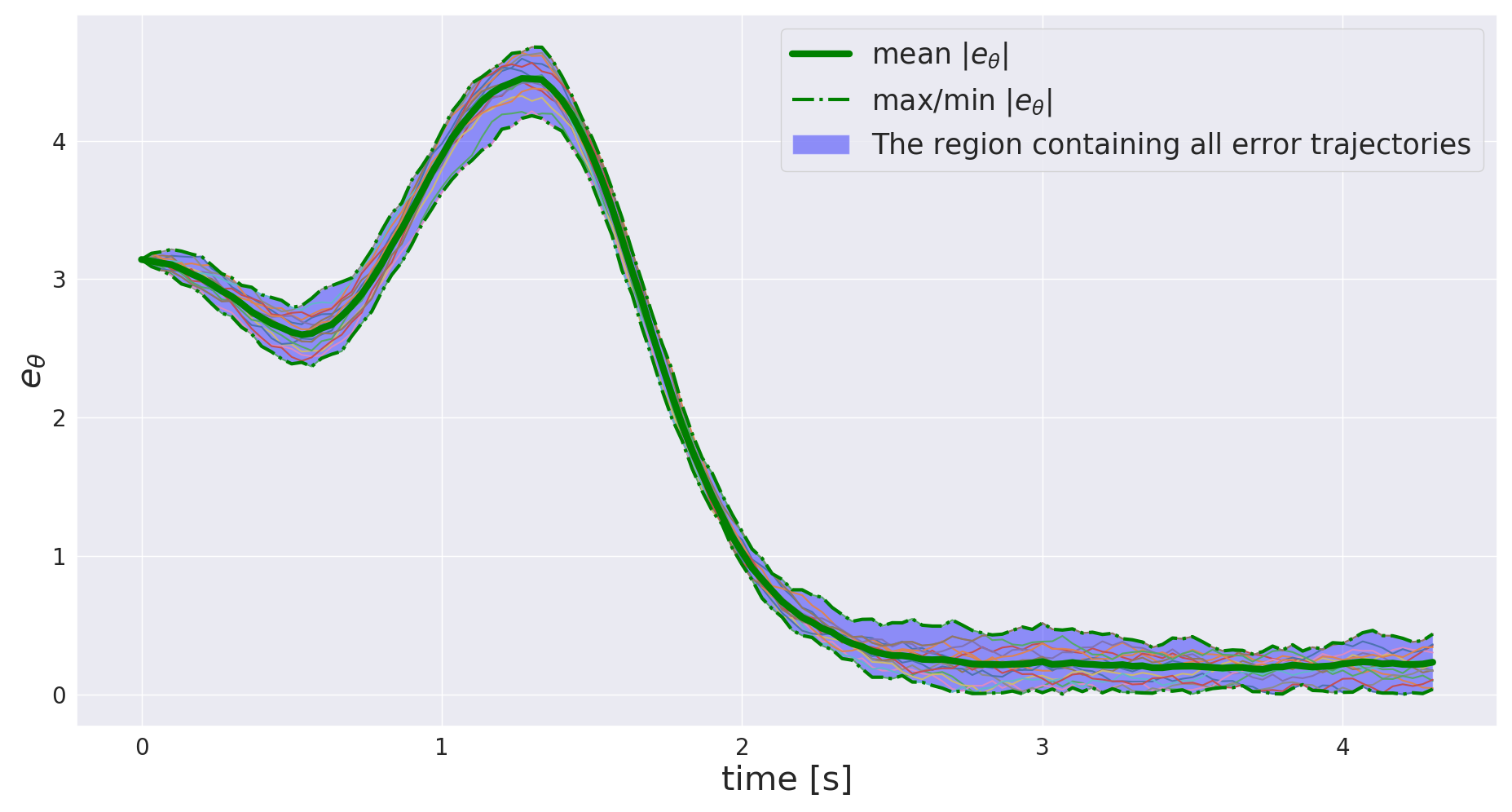}
	\caption{Error statistics for the controlled system using the proposed method on the underactuated pendulum with noise amplitude  }
	\label{fig:error_pendulum}
\vspace{-2 mm}
\end{figure}

\subsection{Results on Ball-and-Beam System}\label{results_ballbeam}
Next, we implemented the proposed method on a ball-and-beam system (shown in Figure~\ref{fig:bb_system})~\cite{8285376}. The ball-and-beam system is a low-dimensional non-linear system with the non-linearity due to the dry friction and delay in the servo motors attached to the table (see Figure~\ref{fig:bb_system}).  
\begin{figure}
    \centering
    \includegraphics[width=0.35\textwidth]{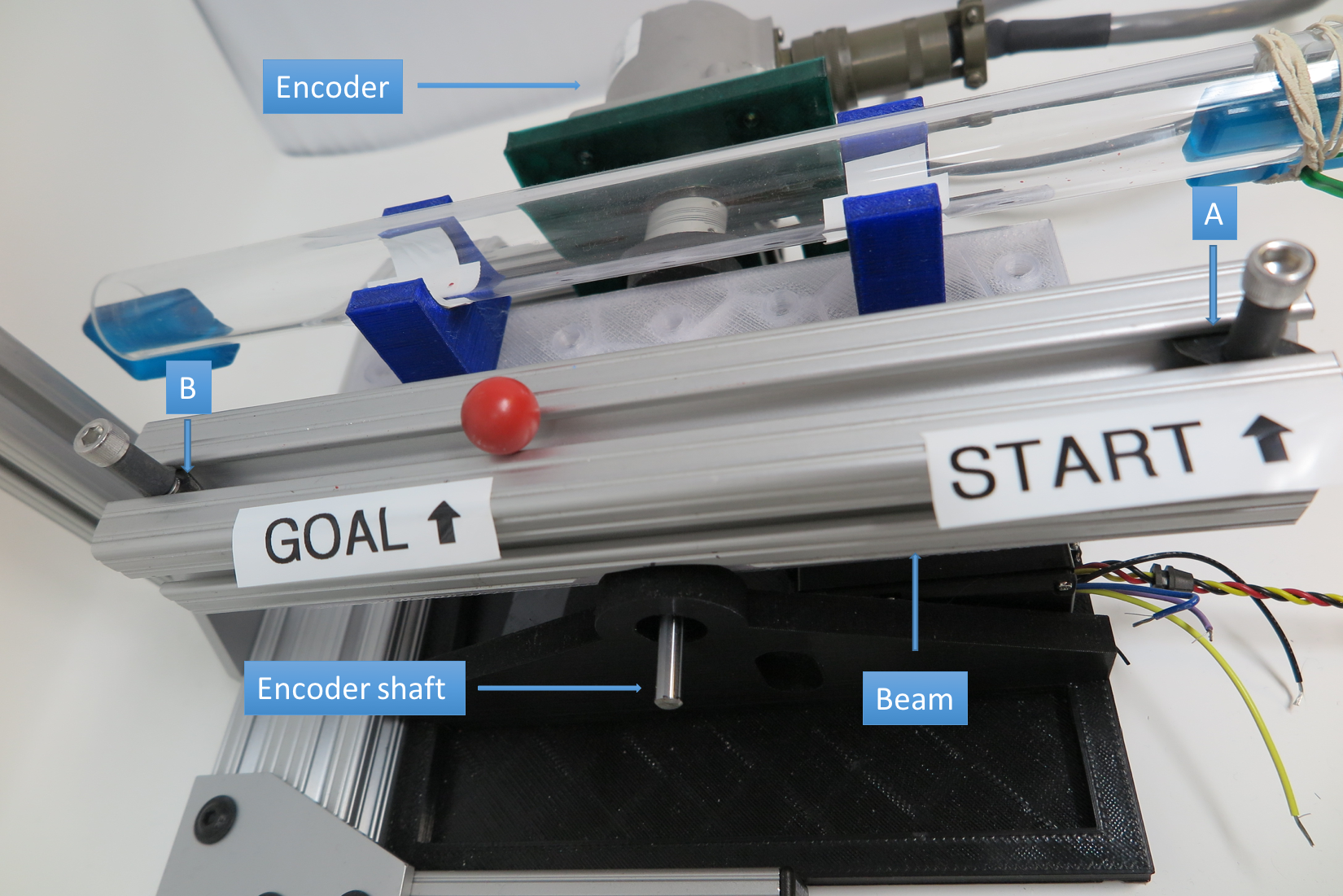}
    \caption{The ball-and-beam system used for the experiments. There is an RGB camera above that measures the location of the ball. The encoder (seen in the figure) measures the angular position of the beam.}
    \label{fig:bb_system}
\end{figure}
The ball-and-beam system can be modeled with 4 state variables $[x, \dot x, \theta, \dot \theta]$, where $x$ is the position of the ball, $\dot x$ is the ball's velocity, $\theta$ is the beam angle in radians, and $\dot \theta$ is the angular velocity of the beam. 
The acceleration of the ball, $\ddot x$, is given by 
\begin{equation}\label{eqn:ball_beam}
    \ddot x = \frac{m_{ball} x \dot \theta^2 - b_1 \dot x - b_2 m_{ball} g \cos(\theta) - m_{ball} g \sin(\theta)}{\frac{I_{ball}}{r^2_{ball}} + m_{ball}},\nonumber
\end{equation}
where $m_{ball}$ is the mass of the ball, $I_{ball}$ is the moment of inertia of the ball, $r_{ball}$ is the radius of the ball, $b_1$ is the coefficient of viscous friction of the ball on the beam, $b_2$ is the coefficient of static (dry) friction of the ball on the beam, and $g$ is the acceleration due to gravity. The beam is actuated by a servo motor (position controlled) and an approximate model for its motion is estimated by fitting an auto-regressive model.
\begin{figure}[!htb] 
	\flushleft
	\includegraphics[width=0.95\columnwidth]{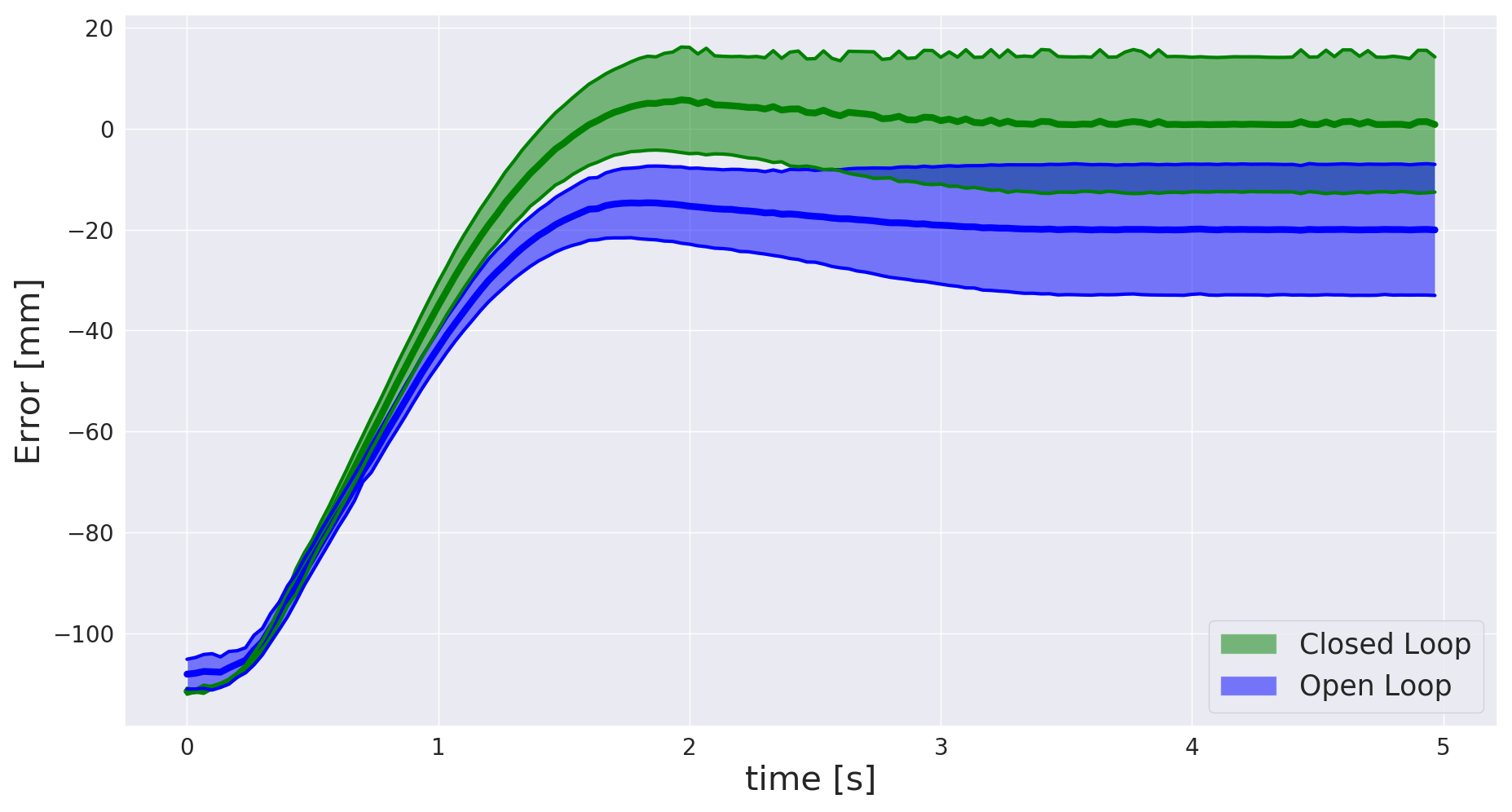}
	\caption{Comparison of the performance of the proposed controller on a ball-and-beam system with the open-loop solution. The plot shows the error in the position of the ball from the regulated position averaged over $12$ runs.}
	\label{fig:ballbeamresult}
\vspace{-2 mm}
\end{figure}
We use this model for the analysis where the ball's rotational inertia is ignored and we approximately estimate the dry friction. The model is inaccurate, as can be seen from the performance of the open-loop controller in Figure~\ref{fig:ballbeamresult}. However, the proposed controller is still able to regulate the ball position at the desired goal showing the stabilizing behavior for the system (see the performance of the closed-loop controller in Figure~\ref{fig:ballbeamresult}). The plot shows the mean and the standard deviation of the error for $12$ runs of the controller. It can be observed that the mean regulation error goes to zero for the closed-loop controller. We believe that the performance of the controller will improve as we improve the model accuracy. In future research, we would like to study the learning behavior for the proposed controller by learning the residual dynamics using GP~\cite{romeres2019semiparametrical}. 
\section{Conclusion and Future Work}\label{sec:conclusions}
This paper presents a method for simultaneously computing an optimal trajectory along with a local, time-invariant stabilizing controller for a dynamical system with known uncertainty bounds. The time-invariant controller was computed by adding a robustness constraint to the trajectory optimization problem. We prove that under certain simplifying assumptions, we can compute the gradient of the robustness constraint so that a gradient-based optimization solver could be used to find a solution for the optimization problem. We tested the proposed approach that shows that it is possible to solve the proposed problem simultaneously. We showed that even a linear parameterization of the stabilizing controller with a linear approximation of the error dynamics allows us to successfully control non-linear systems locally. We tested the proposed method in simulation as well as a physical system. Due to space limitations, we have to skip extra results regarding the behavior of the algorithm.

However, the current approach has two limitations-- it makes linear approximation of dynamics for finding the worst-case deviation, and secondly, the linear parameterization of the stabilizing controller can be limiting for a lot of robotic systems. In future research, we will incorporate these two limitations using Lyapunov theory for non-linear control, for better generalization and more powerful results of the proposed approach. 



\bibliographystyle{IEEEtran}
\bibliography{references}
\end{document}